\documentclass[conference]{IEEEtran}
\IEEEoverridecommandlockouts

\usepackage{amsmath,amssymb,amsfonts,mathtools,bm}
\usepackage{amsthm}
\usepackage{relsize}      
\usepackage{textcomp}

\usepackage{algorithm}
\usepackage{algpseudocode}

\usepackage{graphicx}
\usepackage{svg}
\usepackage{tikz}
\usepackage{float}
\usepackage{placeins}     

\usepackage[table]{xcolor} 
\usepackage{booktabs}     
\usepackage{makecell}     
\usepackage{multirow}      
\usepackage{array}         
\usepackage{colortbl}     
\usepackage{adjustbox}     
\usepackage{tabularx}      
\usepackage{diagbox}      

\usepackage{caption}
\usepackage{subcaption}
\usepackage{sidecap}
\usepackage{setspace}
\usepackage{microtype}     
\usepackage{balance}      
\usepackage{dblfloatfix}  

\usepackage{cite}
\usepackage{url}
\usepackage{bbding}        

\newcommand{\orcidauthorA}{0009-0000-1171-252X} 
\newcommand{\orcidauthorB}{0000-0003-1439-4875} 
\newcommand{\orcidauthorC}{0000-0001-7907-2071} 

\newtheorem{theorem}{Theorem}
\newtheorem{problem}{Problem}

\def\BibTeX{{\rm B\kern-.05em{\sc i\kern-.025em b}\kern-.08em
    T\kern-.1667em\lower.7ex\hbox{E}\kern-.125emX}}

\definecolor{best_red}{HTML}{9A0000}
\definecolor{second_blue}{HTML}{00009B}
\definecolor{gain_up}{HTML}{f08a5d}     
\definecolor{gain_down}{HTML}{00adb5}   

\newcommand{\best}[1]{\textbf{\textcolor{best_red}{#1}}}

\newcommand{\second}[1]{\textcolor{second_blue}{\underline{\mbox{#1}}}}

\newcommand{\timesw}{\scalebox{0.85}{\fontfamily{cmm}\selectfont \textit{\textbf{x}}}}

\newcommand{\up}[1]{\ensuremath{\hspace{1pt}\mathsmaller{\textcolor{gain_up}{\uparrow \text{#1\%}}}}}
\newcommand{\dn}[1]{\ensuremath{\hspace{1pt}\mathsmaller{\textcolor{gain_down}{\downarrow \text{#1\%}}}}}

\allowdisplaybreaks
\renewcommand{\baselinestretch}{1}
\begin{document}

\title{RMPrior: Bridging Propagation Priors and Diffusion Refinement for Efficient Radio Map Construction}

\author{\IEEEauthorblockN{
Zixuan Guo\IEEEauthorrefmark{1},
Xiucheng Wang\IEEEauthorrefmark{2},
Nan Cheng\IEEEauthorrefmark{2}
}
\IEEEauthorblockA{
\IEEEauthorrefmark{1}School of Physics, Xidian University, Xi'an, 710071, China\\
\IEEEauthorrefmark{2}State Key Laboratory of ISN and School of Telecommunications Engineering, Xidian University, Xi'an, 710071, China\\
Email: 24179100067@stu.xidian.edu.cn, xcwang\_1@stu.xidian.edu.cn, dr.nan.cheng@ieee.org
}
}

\maketitle

\begin{abstract}
Diffusion models achieve high-fidelity radio map construction through iterative denoising, yet their sampling cost limits practicality in dynamic wireless systems where radio maps must be refreshed repeatedly. Meanwhile, classical propagation models encode valuable scene-level knowledge that standard diffusion inference discards entirely by initializing from pure Gaussian noise. This paper bridges propagation priors and diffusion refinement through a mid-start sampling strategy. A matched propagation prior is perturbed to an intermediate diffusion timestep, and the pretrained diffusion backbone executes only the remaining reverse steps, focusing computation on multipath-aware refinement rather than full reconstruction from noise. We provide theoretical analysis establishing an upper bound on the initialization gap, a sufficient condition under which truncation improves reconstruction fidelity, and a formal characterization of prior-quality sensitivity under aggressive truncation. Experiments on IRT4HighRes show that, at $P_{\text{start}}=0.5$, the proposed method achieves a $2.01\times$ speedup while simultaneously improving NMSE, RMSE, SSIM, and PSNR over the full-step baseline. A prior-quality ablation across three propagation models of different fidelity confirms that reconstruction quality tracks prior quality, with the sensitivity amplified under shorter reverse trajectories, consistent with the theoretical predictions. These results also suggest that mid-start reconstruction quality can serve as a proxy for ranking the scene-level fidelity of different propagation models.
\end{abstract}


\section{Introduction}
Radio maps provide a spatial representation of large-scale wireless propagation and serve as a foundation for coverage analysis, deployment planning, beam management, and environment-aware network control. Constructing accurate radio maps has therefore received sustained attention in both model-driven and data-driven communities \cite{wang2026tutorial}. Classical approaches rely on geometry-aware propagation models such as ray-based approximations and the dominant path model (DPM) \cite{deschamps1972ray,11278649,dpm}. These physics-based methods encode valuable knowledge about the propagation environment, including dominant reflection and diffraction paths, scene geometry, and large-scale pathloss structure. However, they typically simplify or omit higher-order multipath interactions, which limits their reconstruction fidelity in complex indoor and dense urban scenarios. Learning-based methods, including RadioUNet, deep Gaussian-process estimators, and GAN-based approaches, can achieve higher fidelity by training on measured or simulated radio-map data \cite{levie2021radiounet,wang2020indoor,zhang2023rme}. These models are generally fast at inference time, but they operate as direct regressors or one-shot generators and do not exploit the structured propagation knowledge that physics-based models can provide.
 
Diffusion models offer a more expressive generative framework for radio-map construction. Foundational formulations such as denoising diffusion probabilistic models (DDPM) and denoising diffusion implicit models (DDIM) demonstrate that high-fidelity synthesis can be achieved through iterative denoising \cite{ho2020denoising,song2020denoising}. In the radio-map domain, RadioDiff and RMDM have shown that diffusion-based reconstruction can outperform earlier learning-based methods in terms of fidelity metrics \cite{wang2024radiodiff,jia2025rmdm}, while efficiency-oriented variants explore lighter generative backbones or trajectory reuse across correlated samples \cite{11152929,11282987}. Despite these advances, iterative reverse sampling remains computationally expensive. In dynamic wireless systems where radio maps must be refreshed repeatedly across transmitter locations, candidate placements, or consecutive time slots, this cumulative inference cost becomes a practical bottleneck. Meanwhile, the propagation knowledge encoded in classical models remains largely unused during diffusion inference, which always starts from pure Gaussian noise regardless of whether a coarse but informative prior is already available.
 
This paper bridges the gap between physics-based propagation priors and diffusion-based refinement. The core idea is that a matched radio map generated by DPM already captures the dominant propagation layout for a given scene and transmitter configuration, and this structured knowledge can be injected into the diffusion sampling process to replace the uninformative pure-noise initialization. Specifically, we perturb the matched DPM radio map to an intermediate diffusion timestep and execute reverse denoising only over the remaining trajectory. The diffusion model is thereby used for multipath-aware refinement rather than for reconstructing the entire radio map from scratch. This mid-start strategy is inference-only and does not require retraining the base diffusion model. We instantiate it on the pretrained RMDM \cite{jia2025rmdm} backbone and evaluate on IRT4HighRes, using the full test domain for the main comparison, a sparse subset for operating-range sensitivity analysis, and a 720-sample overlap subset for prior-quality ablation. The prior-quality ablation further compares three propagation models of different fidelity as initialization sources, showing that reconstruction quality tracks prior quality and suggesting that mid-start diffusion can also serve as a proxy for evaluating the scene-level fidelity of different propagation models. The main contributions are as follows.
\begin{enumerate}
    \item We propose an inference-only mid-start diffusion sampling strategy that bridges propagation priors and diffusion refinement for radio map construction. A matched DPM prior is perturbed to an intermediate diffusion timestep and used to initialize the reverse process on the pretrained RMDM backbone, so that the diffusion model focuses on residual multipath refinement rather than full reconstruction from noise.
    \item We evaluate the proposed method on a high-resolution ray-tracing radio map dataset with 7920 test samples. At $P_{\text{start}}=0.5$, the method reduces average inference latency by $2.01\times$ while simultaneously improving NMSE, RMSE, SSIM, and PSNR over the full-step baseline, indicating that the physics-based initialization not only saves computation but also reduces accumulated sampling error.
    \item We conduct a prior-quality ablation across three propagation models of different fidelity on a 720-sample overlap subset. The results show that reconstruction quality degrades consistently as the prior becomes less informative, and this sensitivity increases under more aggressive truncation. This finding highlights the central role of prior quality and suggests that mid-start diffusion reconstruction can serve as an indirect means to assess the scene-level accuracy of different propagation models.
\end{enumerate}

\section{Preliminary and Problem Formulation}

\subsection{Conditional DDPM}

We adopt the conditional DDPM formulation \cite{ho2020denoising}. Let $x_0 \in \mathbb{R}^{H \times W}$ denote the target radio map in image domain, where each pixel encodes the pathloss value at the corresponding spatial location, and let $c$ denote the conditioning information used by the pretrained diffusion backbone. The forward process gradually perturbs a clean radio map into Gaussian noise according to a variance schedule $\{\beta_t\}_{t=1}^{T}$, where $\alpha_t = 1-\beta_t$ and $\bar{\alpha}_t = \prod_{i=1}^{t}\alpha_i$. The forward marginal at timestep $t$ is
\begin{equation}
q(x_t \mid x_0)
=
\mathcal{N}\!\left(
x_t;
\sqrt{\bar{\alpha}_t}\,x_0,
\left(1-\bar{\alpha}_t\right)I
\right).
\label{eq:forward_marginal}
\end{equation}
The denoising network $\epsilon_\theta(x_t,t,c)$ predicts the injected noise under condition $c$ and is trained with the conditional noise-prediction objective
\begin{equation}
\mathcal{L}_{\mathrm{diff}}
=
\mathbb{E}_{x_0,c,t,\epsilon}
\left[
\left\|
\epsilon-\epsilon_\theta(x_t,t,c)
\right\|_2^2
\right],
\label{eq:diff_loss}
\end{equation}
\begin{equation}
x_t
=
\sqrt{\bar{\alpha}_t}\,x_0
+
\sqrt{1-\bar{\alpha}_t}\,\epsilon,
\qquad
\epsilon\sim\mathcal{N}(0,I).
\label{eq:xt_construct}
\end{equation}
Under standard DDPM inference, sampling starts from $x_T \sim \mathcal{N}(0,I)$ and applies the full $T$-step reverse chain to recover an estimate of $x_0$. At each reverse step, the model predicts the noise component and uses it to obtain the posterior mean for the previous timestep. The per-step transition is
\begin{equation}
x_{t-1}
=
\frac{1}{\sqrt{\alpha_t}}
\left(
x_t - \frac{\beta_t}{\sqrt{1-\bar{\alpha}_t}}\,\epsilon_\theta(x_t,t,c)
\right)
+ \sigma_t z,
\label{eq:reverse_step}
\end{equation}
where $z \sim \mathcal{N}(0,I)$ and $\sigma_t$ is the posterior variance. This formulation shows that each reverse step incurs one full forward pass through the denoising network, making the total inference cost proportional to $T$.

\subsection{Problem Statement}
 
Let $x_{\mathrm{p}} \in \mathbb{R}^{H \times W}$ denote a matched propagation prior generated by a physics-based model for the same scene layout, transmitter location, and output grid as the target sample. Such a prior encodes the dominant propagation structure, including large-scale pathloss distribution and primary blocked regions, but omits higher-order multipath interactions. Depending on the available scene information and computational budget, this prior can be produced by geometry-based, empirical, or simplified analytical propagation models, each providing a different level of scene-level fidelity.
 
Standard DDPM inference discards all available physics knowledge by initializing from pure Gaussian noise $x_T \sim \mathcal{N}(0,I)$. This is wasteful when a matched propagation prior is already available, because the early reverse steps must reconstruct large-scale propagation structure that the prior already provides. The goal of this paper is to design an inference-only sampling strategy that exploits the structural information in $x_{\mathrm{p}}$ to reduce the number of reverse denoising steps while preserving reconstruction fidelity. Let $\mathcal{I}: \mathbb{R}^{H \times W} \rightarrow \mathbb{R}^{H \times W}$ denote an initialization mapping that converts the matched prior into a diffusion-compatible starting state, let $\mathcal{R}_\theta(\cdot, t_s, c)$ denote the reverse denoising chain from timestep $t_s$ to $0$ using the pretrained backbone, and let $\mathcal{L}_{\mathrm{recon}}^{\mathrm{full}}$ denote the expected reconstruction loss of the full-chain baseline starting from $x_T \sim \mathcal{N}(0,I)$. The problem is formulated as follows. \begin{problem}
\label{prob:midstart}
\begin{align}
\min_{t_s,\, \mathcal{I}} \quad & t_s \label{eq:prob_obj} \\
\mathrm{s.t.} \quad & \hat{x}_0 = \mathcal{R}_\theta\!\left(\mathcal{I}(x_{\mathrm{p}}),\, t_s,\, c\right), \tag{\ref{eq:prob_obj}a}\label{eq:prob_recon} \\
& \mathbb{E}\!\left[\left\|\hat{x}_0 - x_0\right\|_2^2\right] \leq \mathcal{L}_{\mathrm{recon}}^{\mathrm{full}}, \tag{\ref{eq:prob_obj}b}\label{eq:prob_fidelity} \\
& t_s \in \{1, \dots, T\}.\tag{\ref{eq:prob_obj}c} \label{eq:prob_range}
\end{align}
\end{problem}

\section{Proposed Method}

\begin{figure*}[htbp]
    \centering
    \includegraphics[width=0.75\linewidth]{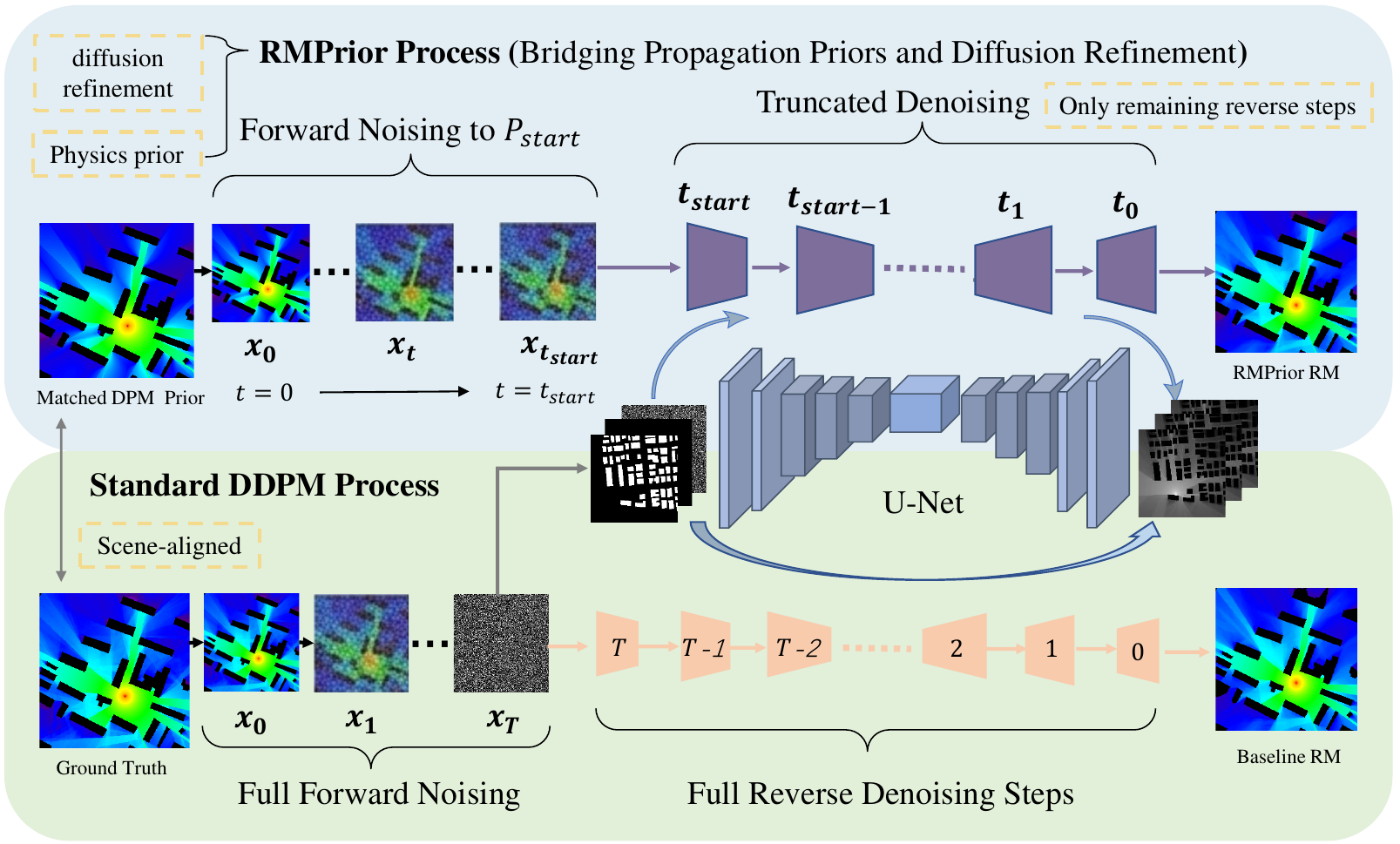}
    \caption{Overview of the proposed RMPrior pipeline. A matched propagation prior is perturbed to an intermediate diffusion state specified by $P_{\text{start}}$, and reverse denoising is executed only over the remaining trajectory. The propagation prior provides the dominant pathloss layout as structured initialization, while the diffusion model refines the multipath details that the prior omits.}
    \label{fig:framework}
\end{figure*}

Fig.~\ref{fig:framework} summarizes the proposed pipeline. The core idea is to bridge the structured propagation knowledge encoded in a physics-based prior with the expressive refinement capacity of a pretrained diffusion model. Rather than starting each reverse trajectory from uninformative Gaussian noise, we construct a scene-aligned intermediate state from the matched prior and let the diffusion model focus on residual multipath refinement.

\subsection{Mid-Start Diffusion Sampling}

We instantiate the initialization mapping $\mathcal{I}$ in Problem~\ref{prob:midstart} as forward noising. In this paper, the matched prior $x_{\mathrm{p}}$ is obtained from the dominant path model (DPM) \cite{dpm}, which provides a coarse but scene-aligned propagation layout. For a chosen start timestep $t_s \in \{1,\dots,T\}$, the mid-start state is constructed as
\begin{equation}
\mathcal{I}(x_{\mathrm{p}}) = x_{t_s}^{\mathrm{p}}
=
\sqrt{\bar{\alpha}_{t_s}}\,x_{\mathrm{p}}
+
\sqrt{1-\bar{\alpha}_{t_s}}\,\epsilon,
\quad
\epsilon\sim\mathcal{N}(0,I).
\label{eq:midstart_state}
\end{equation}
The reverse process is then executed only from $t_s$:
\begin{equation}
x_{t-1}\sim p_\theta(x_{t-1}\mid x_t,c),
\qquad
t=t_s,t_s-1,\dots,1.
\label{eq:midstart_reverse}
\end{equation}
The denoising network and all learned parameters remain unchanged. The only modification is the starting state: the baseline initializes from isotropic noise at $t=T$, whereas the proposed method initializes from a scene-aligned intermediate state at $t=t_s$ that retains the dominant propagation layout of $x_{\mathrm{p}}$. The diffusion model is thereby used to recover multipath-aware details that the prior omits, rather than to reconstruct the entire radio map from scratch.

We parameterize the starting point by a normalized scalar $P_{\text{start}} \in (0,1]$ and map it to the discrete index as $t_s = \lfloor P_{\text{start}} T \rfloor$. Setting $P_{\text{start}}=1.0$ recovers the full-step baseline. Since the denoising network is evaluated once per reverse step, the sampling cost scales approximately linearly with $P_{\text{start}}$, reducing network evaluations from $T$ to $\lfloor P_{\text{start}} T \rfloor$. This parameter controls how the reconstruction workload is divided between the physics-based prior and the learned refinement. The experiments in Section~\ref{sec:experiments} make this tradeoff directly observable.

\subsection{Theoretical Analysis}

We provide three formal results. Theorem~\ref{thm:init_gap} characterizes the initialization gap between the mid-start state and the ideal forward-noised ground truth. Theorem~\ref{thm:truncation_benefit} establishes when truncation reduces the total reconstruction error. Theorem~\ref{thm:sensitivity} shows that prior-quality sensitivity increases under more aggressive truncation.

\subsubsection{Initialization Gap}

Let $x_{t_s}^{\mathrm{GT}} = \sqrt{\bar{\alpha}_{t_s}}\,x_0 + \sqrt{1-\bar{\alpha}_{t_s}}\,\epsilon$ denote the ideal intermediate state obtained by forward noising the ground truth $x_0$ to timestep $t_s$. The mid-start state $x_{t_s}^{\mathrm{p}}$ in Eq.~(\ref{eq:midstart_state}) differs from $x_{t_s}^{\mathrm{GT}}$ due to the mismatch between $x_{\mathrm{p}}$ and $x_0$.

\begin{theorem}[Initialization gap bound]
\label{thm:init_gap}
Under the same noise realization $\epsilon$,
\begin{equation}
\left\| x_{t_s}^{\mathrm{p}} - x_{t_s}^{\mathrm{GT}} \right\|_2
=
\sqrt{\bar{\alpha}_{t_s}}
\left\| x_{\mathrm{p}} - x_0 \right\|_2.
\label{eq:init_gap}
\end{equation}
\end{theorem}

\begin{proof}
Expanding both states and canceling the shared noise term yields
\begin{equation}
x_{t_s}^{\mathrm{p}} - x_{t_s}^{\mathrm{GT}}
= \sqrt{\bar{\alpha}_{t_s}}\left(x_{\mathrm{p}} - x_0\right).
\end{equation}
Taking the $\ell_2$ norm gives the result.
\end{proof}

Theorem~\ref{thm:init_gap} shows that the initialization gap is proportional to the prior mismatch $\|x_{\mathrm{p}} - x_0\|_2$, scaled by $\sqrt{\bar{\alpha}_{t_s}}$. Since $\sqrt{\bar{\alpha}_{t_s}}$ increases as $t_s$ decreases, smaller $P_{\text{start}}$ preserves more of the prior mismatch while larger $P_{\text{start}}$ attenuates it through heavier noising.

\subsubsection{Condition for Truncation Benefit}

We assume the reverse chain is a contraction mapping where each step reduces the distance to the ideal trajectory by a factor $\gamma_t \in (0,1)$, and let $\delta_t$ denote the per-step denoiser approximation error at timestep $t$. Let $\hat{x}_0^{\mathrm{full}}$ and $\hat{x}_0^{\mathrm{mid}}$ denote the reconstructions from the full-chain and mid-start sampling, respectively.

\begin{theorem}[Truncation benefit condition]
\label{thm:truncation_benefit}
Mid-start sampling yields lower reconstruction error than the full-chain baseline when
\begin{equation}
\sqrt{\bar{\alpha}_{t_s}} \left\| x_{\mathrm{p}} - x_0 \right\|_2
\cdot \prod_{t=1}^{t_s} \gamma_t
<
\sum_{t=t_s+1}^{T} \delta_t \cdot \prod_{j=1}^{t-1}\gamma_j.
\label{eq:benefit_condition}
\end{equation}
\end{theorem}

\begin{proof}
Under the contraction assumption, the mid-start error bound is
\begin{equation}
\|\hat{x}_0^{\mathrm{mid}} - x_0\|_2
\leq
\sqrt{\bar{\alpha}_{t_s}} \|x_{\mathrm{p}} - x_0\|_2 \prod_{t=1}^{t_s}\gamma_t
+
\sum_{t=1}^{t_s} \delta_t \prod_{j=1}^{t-1}\gamma_j.
\label{eq:mid_error}
\end{equation}
Since $\bar{\alpha}_T \approx 0$, the full-chain starts from uninformative noise, giving
\begin{equation}
\|\hat{x}_0^{\mathrm{full}} - x_0\|_2
\leq
\sum_{t=1}^{T} \delta_t \prod_{j=1}^{t-1}\gamma_j.
\label{eq:full_error}
\end{equation}
Subtracting Eq.~(\ref{eq:mid_error}) from Eq.~(\ref{eq:full_error}), mid-start is beneficial when the accumulated error of the removed steps exceeds the contracted initialization gap, yielding Eq.~(\ref{eq:benefit_condition}).
\end{proof}

The left-hand side of Eq.~(\ref{eq:benefit_condition}) represents the contracted initialization gap from the physics prior. The right-hand side is the accumulated denoiser error from steps $t_s{+}1$ to $T$ that truncation eliminates. This condition is more easily satisfied when the prior is close to the ground truth and when per-step denoiser error is non-negligible in the early, heavily noised steps. It explains the empirical observation that $P_{\text{start}}=0.5$ improves fidelity over the full-step baseline.

\subsubsection{Prior-Quality Sensitivity Under Truncation}

Consider two priors $x_{\mathrm{p}}^{(1)}$ and $x_{\mathrm{p}}^{(2)}$ with mismatch levels $d_1 = \|x_{\mathrm{p}}^{(1)} - x_0\|_2 < d_2 = \|x_{\mathrm{p}}^{(2)} - x_0\|_2$.

\begin{theorem}[Sensitivity amplification]
\label{thm:sensitivity}
Under the same contraction assumption, the reconstruction error difference between the two priors satisfies
\begin{equation}
\Delta E(t_s)
\leq
\sqrt{\bar{\alpha}_{t_s}}
\left(d_2 - d_1\right)
\prod_{t=1}^{t_s}\gamma_t.
\label{eq:sensitivity}
\end{equation}
Moreover, $\Delta E(t_s)$ increases monotonically as $t_s$ decreases.
\end{theorem}

\begin{proof}
From Eq.~(\ref{eq:mid_error}), the two priors share the same accumulated denoiser error term. Their error bound difference is
\begin{equation}
\Delta E(t_s)
\leq
\sqrt{\bar{\alpha}_{t_s}} (d_2 - d_1) \prod_{t=1}^{t_s}\gamma_t.
\end{equation}
Decreasing $t_s$ increases $\sqrt{\bar{\alpha}_{t_s}}$ and increases $\prod_{t=1}^{t_s}\gamma_t$ because fewer contraction steps are applied. Both factors enlarge $\Delta E(t_s)$.
\end{proof}

Theorem~\ref{thm:sensitivity} formalizes why prior quality becomes more critical under aggressive truncation. At moderate $P_{\text{start}}$, the reverse chain has enough contraction steps to absorb the initialization difference between a high-fidelity and a low-fidelity prior. At small $P_{\text{start}}$, fewer contraction steps are available and the scaling factor $\sqrt{\bar{\alpha}_{t_s}}$ is larger, amplifying the performance gap. This result provides theoretical support for the prior-quality ablation in the experiments.

\begin{figure*}[!t]
\centering
\captionsetup{font={small}, skip=8pt}
\renewcommand{\arraystretch}{0.8}
\setlength{\tabcolsep}{1.2pt}

\begin{tabular}{c c c c c c}
  \includegraphics[width=0.138\linewidth]{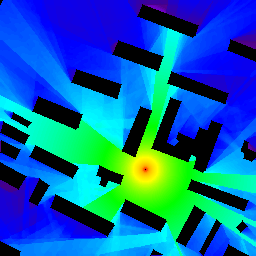} &
  \includegraphics[width=0.138\linewidth]{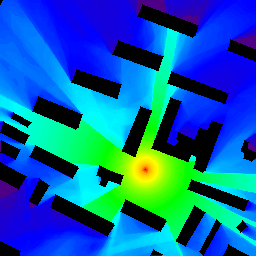} &
  \includegraphics[width=0.138\linewidth]{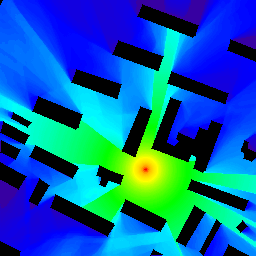} &
  \includegraphics[width=0.138\linewidth]{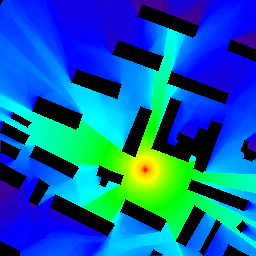} &
  \includegraphics[width=0.138\linewidth]{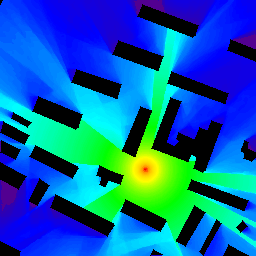} &
  \includegraphics[width=0.138\linewidth]{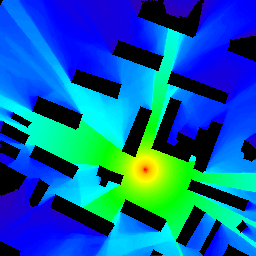} \\

  \includegraphics[width=0.138\linewidth]{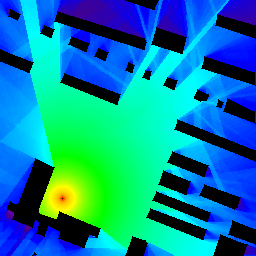} &
  \includegraphics[width=0.138\linewidth]{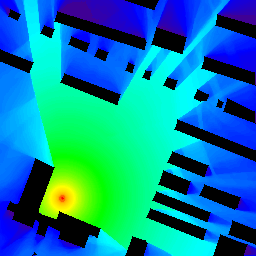} &
  \includegraphics[width=0.138\linewidth]{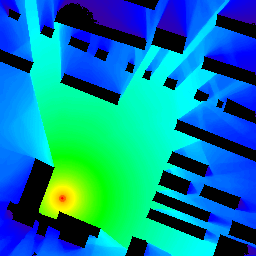} &
  \includegraphics[width=0.138\linewidth]{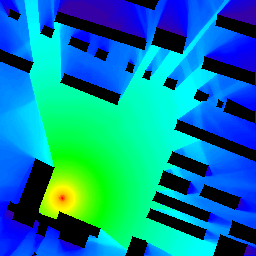} &
  \includegraphics[width=0.138\linewidth]{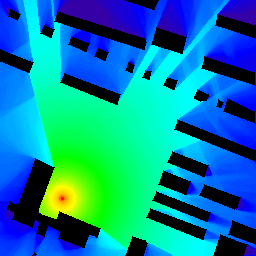} &
  \includegraphics[width=0.138\linewidth]{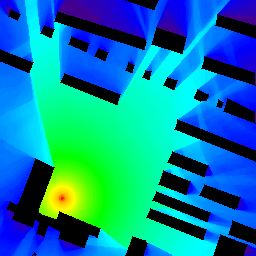} \\

  \includegraphics[width=0.138\linewidth]{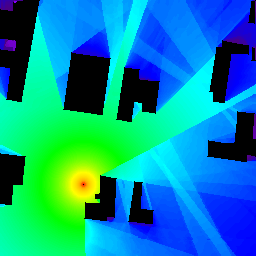} &
  \includegraphics[width=0.138\linewidth]{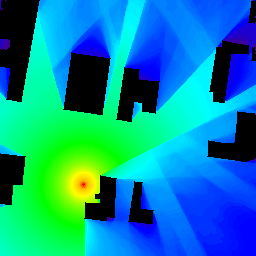} &
  \includegraphics[width=0.138\linewidth]{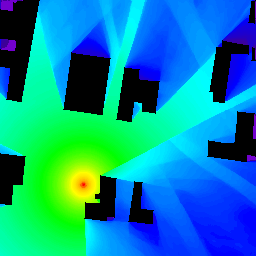} &
  \includegraphics[width=0.138\linewidth]{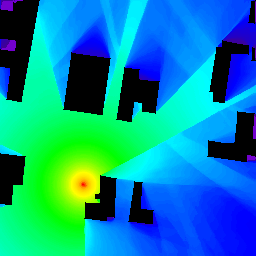} &
  \includegraphics[width=0.138\linewidth]{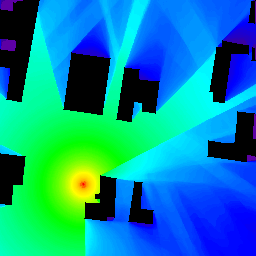} &
  \includegraphics[width=0.138\linewidth]{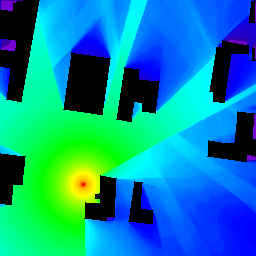} \\

  \includegraphics[width=0.138\linewidth]{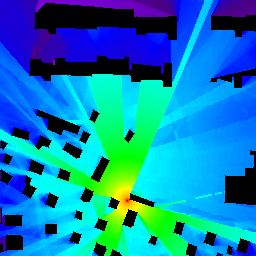} &
  \includegraphics[width=0.138\linewidth]{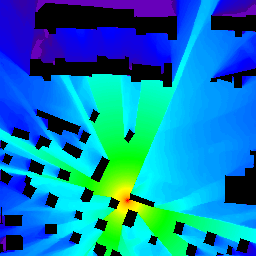} &
  \includegraphics[width=0.138\linewidth]{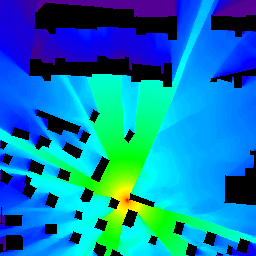} &
  \includegraphics[width=0.138\linewidth]{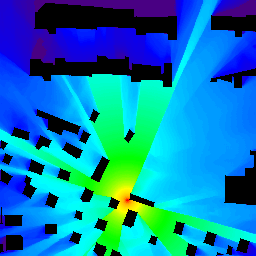} &
  \includegraphics[width=0.138\linewidth]{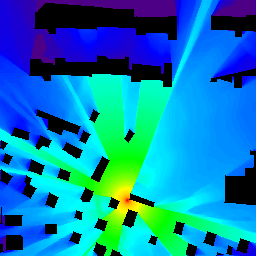} &
  \includegraphics[width=0.138\linewidth]{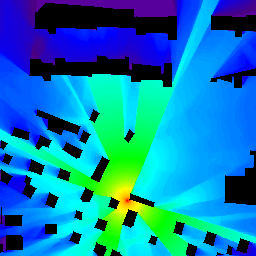} \\

  \small\bfseries Ground Truth &
  \small\bfseries $P_{\text{start}}=0.50$ &
  \small\bfseries $P_{\text{start}}=0.40$ &
  \small\bfseries $P_{\text{start}}=0.30$ &
  \small\bfseries $P_{\text{start}}=0.25$ &
  \small\bfseries Baseline
\end{tabular}

\caption{Visual comparison under different $P_{\text{start}}$ settings. Moderate starts preserve global propagation layout and local attenuation details more reliably than aggressive truncation.}
\label{fig:pstart_visual_grid}
\end{figure*}
\section{Experimental Results}
\label{sec:experiments}

\subsection{Experimental Settings}
We evaluate the proposed method on IRT4HighRes, a high-resolution radio map dataset generated by advanced ray-tracing simulation that captures complex multipath propagation in realistic urban environments. The dataset contains 99 scenes with 80 transmitter locations each, yielding 7920 test samples. The pretrained RMDM backbone \cite{jia2025rmdm} is used with $T=100$ timesteps under a fixed DDPM sampling schedule. The baseline corresponds to $P_{\text{start}}=1.0$, which exactly recovers the full-step RMDM sampling from pure Gaussian noise. RME-GAN \cite{zhang2023rme} and RadioUNet \cite{levie2021radiounet} are included as representative existing methods. All runtimes are measured on a single NVIDIA RTX 4090 GPU. We report NMSE, RMSE, SSIM, and PSNR to assess reconstruction quality. A sparse IRT4 subset is used for $P_{\text{start}}$ sensitivity analysis. A 720-sample overlap subset formed by nine shared scene IDs between IRT4HighRes and the auxiliary ke \cite{giovaneli1984analysis}/cost \cite{correia2009view} prior datasets is used for prior-quality ablation, where the matched DPM prior $x_{\mathrm{p}}$ is replaced by the lower-quality ke and cost priors while the diffusion backbone and reverse-step budget remain fixed.

\begin{table*}[!t]
\centering
\caption{Quantitative results on IRT4HighRes. \best{Bold Red} and \second{Underlined Blue} indicate the best and second-best results. $P_{\text{start}}=1.0$ exactly recovers the full-step RMDM baseline. Speedup and PSNR gain are reported relative to RMDM (35.28\,dB, 514\,ms).}
\label{tab:main_results}
\renewcommand{\arraystretch}{1.15}
\resizebox{0.75\linewidth}{!}{
\begin{tabular}{@{}l|c|c|c|c|c|c|c|c@{}}
\toprule[1.5pt]
\textbf{Method} & $P_{\text{start}}$ & \textbf{NMSE} $\downarrow$ & \textbf{RMSE} $\downarrow$ & \textbf{SSIM} $\uparrow$ & \textbf{PSNR} (dB) $\uparrow$ & \textbf{Time} (ms) $\downarrow$ & \textbf{Speedup} (\timesw) & \textbf{Gain} (PSNR) \\
\midrule[1pt]
RME-GAN   & -- & 0.01150 & 0.03030 & 0.93230 & 30.54 & \best{42.0} & -- & -- \\
RadioUNet & -- & 0.05904 & 0.04796 & 0.82094 & 26.73 & \second{60.0} & -- & -- \\
RMDM      & -- & 0.00688 & 0.01808 & 0.93910 & 35.28 & 514.0 & -- & -- \\
\midrule
              & 1.00 & 0.00688 & 0.01808 & 0.93910 & 35.28 & 514.0 & \timesw 1.00 & -- \\
              & 0.90 & 0.00686 & 0.01806 & 0.93922 & 35.29 & 465.0 & \timesw 1.11 & \up{0.03} \\
              & 0.80 & 0.00686 & 0.01806 & 0.93916 & 35.28 & 418.0 & \timesw 1.23 & -- \\
              & 0.70 & \second{0.00685} & 0.01804 & 0.93921 & 35.30 & 359.0 & \timesw 1.43 & \up{0.06} \\
\textbf{Ours} & 0.60 & \second{0.00685} & \second{0.01802} & \second{0.93928} & 35.33 & 313.0 & \timesw 1.64 & \up{0.14} \\
              & 0.50 & \best{0.00683} & \best{0.01794} & \best{0.93935} & \second{35.39} & 256.0 & \timesw 2.01 & \up{0.31} \\
              & 0.40 & 0.00691 & 0.01810 & 0.93917 & \best{35.41} & 206.0 & \timesw 2.50 & \up{0.37} \\
              & 0.30 & 0.00723 & 0.01838 & 0.93846 & 35.21 & 154.0 & \second{\timesw 3.34} & \dn{0.20} \\
              & 0.25 & 0.00751 & 0.01871 & 0.93789 & 35.05 & 138.0 & \best{\timesw 3.72} & \dn{0.65} \\
\bottomrule[1.5pt]
\end{tabular}}
\end{table*}

\subsection{Main Results on IRT4HighRes}
 
Table~\ref{tab:main_results} reports the main results. At $P_{\text{start}}=0.5$, the proposed method reduces latency from 514\,ms to 256\,ms ($2.01\times$ speedup) while improving NMSE from 0.00688 to 0.00683, RMSE from 0.01808 to 0.01794, SSIM from 0.93910 to 0.93935, and PSNR from 35.28\,dB to 35.39\,dB. This simultaneous improvement in speed and fidelity is consistent with Theorem~\ref{thm:truncation_benefit}: the DPM prior provides a sufficiently informative initialization that the contracted initialization gap is smaller than the accumulated denoiser error from the removed early steps. Moderate starts in $P_{\text{start}}\in[0.4,0.6]$ preserve the best speed--quality balance, with $P_{\text{start}}=0.5$ offering the best aggregate fidelity. Beyond this range, $P_{\text{start}}=0.3$ and $0.25$ further reduce runtime but error metrics begin to rise, as the remaining reverse chain becomes too short to contract the initialization gap.

\subsection{Operating-Range Sensitivity}

Table~\ref{tab:sparse_sensitivity} reports the sensitivity to $P_{\text{start}}$ on the sparse IRT4 subset. These values are reported on a different subset and are not directly comparable to the full-domain results. The same pattern persists: $P_{\text{start}}=0.5$ achieves the lowest NMSE and RMSE, while $P_{\text{start}}=0.4$ offers the highest SSIM at the cost of slightly worse error metrics. More aggressive starts at $P_{\text{start}}=0.3$ and $0.25$ show consistent degradation across all metrics. This subset experiment confirms the design choice from the main evaluation: moderate truncation provides the most favorable tradeoff, while overly aggressive truncation leaves insufficient denoising budget for detail recovery.

\begin{table}[!t]
\centering
\small
\caption{$P_{\text{start}}$ sensitivity on the sparse IRT4 subset. \best{Bold Red} and \second{Underlined Blue} indicate the best and second-best results.}
\label{tab:sparse_sensitivity}
\renewcommand{\arraystretch}{1.15}
\resizebox{0.75\linewidth}{!}{
\begin{tabular}{@{}l|c|c|c|c|c@{}}
\toprule[1.0pt]
$P_{\text{start}}$ & \textbf{NMSE} $\downarrow$ & \textbf{RMSE} $\downarrow$ & \textbf{SSIM} $\uparrow$ & \textbf{PSNR} (dB) $\uparrow$ & \textbf{Time} (ms) $\downarrow$ \\
\midrule
0.60 & \second{0.00903} & \second{0.02058} & 0.93110 & \second{33.79} & 312 \\
0.50 & \best{0.00895}   & \best{0.02053}   & \second{0.93218} & \best{33.82}   & 259 \\
0.40 & 0.00928          & 0.02096          & \best{0.94210}   & 33.77          & 207 \\
0.30 & 0.00952          & 0.02114          & 0.92439          & 33.44          & \second{158} \\
0.25 & 0.00959          & 0.02134          & 0.92224          & 33.28          & \best{139} \\
\bottomrule[1.0pt]
\end{tabular}}
\end{table}

\subsection{Prior-Quality Ablation}
 
Table~\ref{tab:prior_ablation} isolates the role of prior quality on the 720-sample overlap subset. The matched DPM prior achieves the best results at both $P_{\text{start}}=0.5$ and $P_{\text{start}}=0.3$, the cost prior \cite{correia2009view} causes a moderate degradation, and the ke prior \cite{giovaneli1984analysis} causes a much larger drop. At $P_{\text{start}}=0.5$, replacing DPM with cost degrades PSNR from 36.42\,dB to 35.09\,dB, while ke drops it to 32.99\,dB. At $P_{\text{start}}=0.3$, the gap widens sharply: DPM still reaches 36.16\,dB, but cost falls to 30.35\,dB and ke falls to 25.54\,dB. The runtime remains nearly unchanged across prior types at a fixed $P_{\text{start}}$, confirming that the performance gap is driven by initialization quality rather than by computational budget.
 
This amplification effect is precisely what Theorem~\ref{thm:sensitivity} predicts. At smaller $P_{\text{start}}$, the scaling factor $\sqrt{\bar{\alpha}_{t_s}}$ is larger and fewer contraction steps are available, so the reconstruction error difference $\Delta E(t_s)$ between a high-fidelity and a low-fidelity prior is magnified. The ranking DPM $>$ cost $>$ ke is consistent across both tested $P_{\text{start}}$ values, and the widening gap under more aggressive truncation confirms the monotonicity result in Theorem~\ref{thm:sensitivity}. An additional observation is that the reconstruction quality ranking among the three priors reflects their respective levels of scene-level modeling fidelity, suggesting that mid-start diffusion reconstruction can serve as a proxy metric for evaluating how well different propagation models capture the dominant spatial structure of a given scene.

\begin{table}[!t]
\centering
\small
\caption{Prior-quality ablation on the 720-sample overlap subset (scene IDs: 615, 647, 654, 662, 664, 681, 684, 695, 698). \best{Bold Red} and \second{Underlined Blue} indicate the best and second-best results.}
\label{tab:prior_ablation}
\renewcommand{\arraystretch}{1.15}
\resizebox{0.75\linewidth}{!}{
\begin{tabular}{@{}l|l|c|c|c|c|c@{}}
\toprule[1.2pt]
\textbf{Prior} & $P_{\text{start}}$ & \textbf{NMSE} $\downarrow$ & \textbf{RMSE} $\downarrow$ & \textbf{SSIM} $\uparrow$ & \textbf{PSNR} (dB) $\uparrow$ & \textbf{Time} (ms) $\downarrow$ \\
\midrule
\multirow{2}{*}{DPM}
  & 0.50 & \best{0.00494} & \best{0.01595} & \best{0.95173} & \best{36.42} & 257.0 \\
  & 0.30 & \second{0.00544} & \second{0.01654} & \second{0.95043} & \second{36.16} & \second{154.2} \\
\midrule
\multirow{2}{*}{ke}
  & 0.50 & 0.00846 & 0.02321 & 0.94475 & 32.99 & 256.9 \\
  & 0.30 & 0.04602 & 0.05636 & 0.89074 & 25.54 & \best{154.0} \\
\midrule
\multirow{2}{*}{cost}
  & 0.50 & 0.00585 & 0.01816 & 0.94974 & 35.09 & 259.6 \\
  & 0.30 & 0.01754 & 0.03304 & 0.93447 & 30.35 & 155.8 \\
\bottomrule[1.2pt]
\end{tabular}}
\end{table}

\subsection{Visual Analysis}

Fig.~\ref{fig:dpm_prior_visual} illustrates the role of the matched DPM prior. The prior captures the coarse propagation layout, including dominant blocked regions and large-scale pathloss distribution, but misses the fine intensity details needed for an accurate radio map. Starting reverse denoising from this structured state allows the diffusion model to concentrate its updates on multipath refinement rather than on reconstructing the entire map from noise. This visual evidence is consistent with the quantitative gains in Table~\ref{tab:main_results} and the prior-quality sensitivity in Table~\ref{tab:prior_ablation}.

Fig.~\ref{fig:pstart_visual_grid} visualizes the effect of varying $P_{\text{start}}$. At $P_{\text{start}}=0.5$ and $0.4$, the reconstructed maps remain close to the ground truth in both global propagation contours and local attenuation transitions. At $P_{\text{start}}=0.3$ and $0.25$, visible deviations become more frequent in fine-grained intensity patterns. The visual trend matches the numerical behavior in Table~\ref{tab:sparse_sensitivity}: a moderately noised prior provides sufficient structure to accelerate sampling, whereas overly aggressive truncation leaves too little denoising budget for detail recovery.

\begin{figure}[!t]
\centering
\captionsetup{font={small}, skip=8pt}
\renewcommand{\arraystretch}{0.8}
\setlength{\tabcolsep}{2pt}

\begin{tabular}{c c c}
  \includegraphics[width=0.25\linewidth]{figs/irt4/gt/650_1.png} &
  \includegraphics[width=0.25\linewidth]{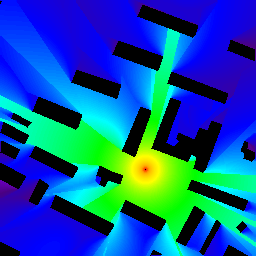} &
  \includegraphics[width=0.25\linewidth]{figs/irt4/p5/650_1.png} \\

  \includegraphics[width=0.25\linewidth]{figs/irt4/gt/696_0.png} &
  \includegraphics[width=0.25\linewidth]{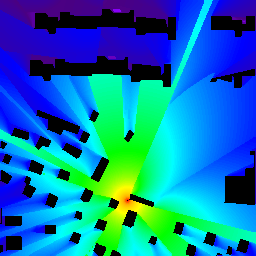} &
  \includegraphics[width=0.25\linewidth]{figs/irt4/p5/696_0.png} \\

  \small\bfseries Ground Truth &
  \small\bfseries Matched DPM Prior &
  \small\bfseries Reconstruction
\end{tabular}

\caption{Visual comparison on IRT4HighRes. The DPM prior captures the coarse layout, while mid-start denoising at $P_{\text{start}}=0.5$ recovers high-fidelity details.}
\label{fig:dpm_prior_visual}
\end{figure}

\section{Conclusion}
In this paper, we have proposed a mid-start diffusion sampling strategy that bridges physics-based propagation priors and learned diffusion refinement for radio map construction, and have provided theoretical analysis showing that the reconstruction benefit depends on both prior quality and the truncation depth. We have validated this strategy on a high-resolution ray-tracing dataset with 7920 samples, demonstrating a $2.01\times$ inference speedup at $P_{\text{start}}=0.5$ with simultaneous fidelity improvement, and confirmed through a prior-quality ablation that reconstruction quality consistently tracks the scene-level accuracy of the initialization source. The proposed method is particularly suited to dynamic wireless networks where radio maps must be refreshed repeatedly across transmitter locations or time-varying scenarios, as it substantially reduces per-update latency without requiring model retraining. Future work will explore adaptive selection of $P_{\text{start}}$ based on scene complexity and extension to other generative backbones beyond DDPM.

\bibliographystyle{IEEEtran}
\bibliography{ref}

\end{document}